\documentclass{article} 
\usepackage{nips13submit_e,times}
\usepackage{hyperref}
\usepackage{url}

\usepackage{graphicx} 
\usepackage{subfigure} 
\usepackage{algorithm,algorithmicx,algpseudocode}

\usepackage{amsmath,amsfonts,amssymb,amsthm}
\usepackage{bbm,bm}

\newtheorem{lemma}{Lemma}

\newtheorem{definition}{Definition}
\newtheorem{proposition}{Proposition}

\newtheorem{condition}{Condition}

\newcommand{\rR}{I\!\!R}

\DeclareMathOperator{\eE}{\mathbb{E}}

\DeclareMathOperator{\diag}{diag}

\def\numofproj{{P}}

\title{Necessary and Sufficient Conditions for Novel Word Detection in Separable Topic Models }

\author{
Weicong Ding , ~Prakash Ishwar, ~Mohammad H.Rohban, ~Venkatesh Saligrama 
\\
Dept. of Electrical and Computer Engineering\\
Boston University\\
\texttt{$\lbrace$dingwc,pi,mhrohban,srv$\rbrace$@bu.edu} \\
}

\nipsfinalcopy 

\begin{document}

\maketitle

\vspace{-5ex}
\begin{abstract}
The simplicial condition and other stronger conditions that imply it
have recently played a central role in developing polynomial
time algorithms with provable asymptotic consistency and sample
complexity guarantees for topic estimation in {\it separable} topic models
. Of these algorithms, those that rely solely on the simplicial
condition are impractical while the practical ones need stronger
conditions. In this paper, we demonstrate, for the first time, that
the simplicial condition is a fundamental, algorithm-independent,
information-theoretic necessary condition for consistent separable
topic estimation. Furthermore, under solely the simplicial condition,
we present a practical quadratic-complexity algorithm based on random
projections which consistently detects all novel words of all topics
using only up to second-order empirical word moments. This algorithm
is amenable to distributed implementation making it attractive for
``big-data'' scenarios involving a network of large distributed
databases.
%
%
%
%
%
\end{abstract}

\vspace{-3ex}
\section{Introduction}
\vspace{-2ex}
A series of powerful practical algorithms for probabilistic topic
modeling have emerged in the past decade since the seminal work on
Latent Dirichlet Allocation (LDA) \cite{LDA:ref}. This has propelled
topic modeling into a popular tool for learning latent semantic
structures in large datasets.
Formally, topic models consider a collection of $M$ documents, each
modeled as being generated by $N$ iid drawings of words from an
unknown $W\times 1$ document word-distribution vector over a
vocabulary of size $W$. By positing $K$ latent {\it topics} as
distribution vectors over the vocabulary, each document
word-distribution vector arises as a {\it probabilistic mixture} of
the $K$ topic vectors. The topic proportions for documents are assumed
to be sampled in an iid manner from some prior distribution such as
the Dirichlet distribution in LDA.

For future reference, let $\bm\beta$ denote the unknown $W\times K$
topic matrix whose columns are the $K$ latent topics, $\bm\theta$ the
$K \times M$ probabilistic topic-weight matrix whose columns are topic
proportions of the $M$ documents, and let $\mathbf{X}$ denote the
$W\times M$ {\it empirical} word-by-document matrix whose columns are
word-frequency vectors of the $M$ documents.
Typically, $W \gg K$.

While the prevailing approach is to find a maximum likelihood fit of
$\mathbf{X}$ to the generative model through approximations or
heuristics, a recent trend has been to develop topic estimation
algorithms with provable guarantees under suitable additional
conditions \cite{ARORA:ref,Arora2:ref,DDP:ref, Anan13:ref}. Chief
among them is the so-called topic {\it separability} condition 
\cite{Donhunique:ref, ARORA:ref, DDP:ref}:
%
%
\begin{condition}({\it Topic separability}) \label{separable_condition}
A topic matrix ${\bm \beta}\in\rR^{W\times K}$ is separable if for
each topic $k$, there exists some word $i$ such that ${\bm
  \beta}_{i,k}>0$ and ${\bm\beta}_{i,l}=0$, $\forall ~ l\neq k$.
\end{condition}
The words that are unique to each topic, referred to as ``novel
words'', are key to the recovery of topics. However, as implicitly
suggested in \cite{Donhunique:ref}, and illustrated in
Fig.~\ref{counterexample}, separability alone does not guarantee the
uniqueness of recovery.
\begin{figure}[!hbt]
\label{counterexample}
\centering
\begin{tabular}{ccccc}
$
\begin{pmatrix}
1&0&0\\
0&1&0\\
0&0&1\\
0&0&1\\
&\ldots &
\end{pmatrix}
$
& 
$ \begin{pmatrix}
\leftarrow & \bm{\theta}_1& \rightarrow \\
\leftarrow & \bm{\theta}_2 & \rightarrow \\
\leftarrow & 0.5\bm{\theta}_1 +0.5\bm{\theta}_2& \rightarrow 
\end{pmatrix} $
&
=
&
$ \begin{pmatrix}
1&0&0\\
0&1&0\\
0&0&1\\
0.5&0.5&0\\
&\ldots &
\end{pmatrix} $
&
$ \begin{pmatrix}
\leftarrow & \bm{\theta}_1& \rightarrow \\
\leftarrow & \bm{\theta}_2 & \rightarrow \\
\leftarrow & 0.5\bm{\theta}_1 +0.5\bm{\theta}_2& \rightarrow 
\end{pmatrix} $ \\
\\
$\bm{\beta}_1 $& $\bm{\theta}$ & & $\bm{\beta}_2 $ & $\bm{\theta}$
\end{tabular}
\vglue -2ex
\caption{Example showing that separability {\it alone} does not
  guarantee uniqueness of decomposition. Here, $\bm \beta_1 \bm
  \theta = \bm \beta_2 \bm \theta$ are two valid decompositions
  where $\bm\beta_1$ and $\bm\beta_2$ are both separable.}
\end{figure}
Therefore to develop algorithms with provable (asymptotic) consistency
guarantees ($N$ fixed, $M \rightarrow\infty$), a number of recent
papers have imposed additional conditions on the prior distribution of
the columns of $\bm\theta$. This is summarized in
Table~\ref{relatedworks} where $\mathbf{a}$ and $\mathbf{R}$ are,
respectively, the expectation and correlation matrix of the prior on
the columns of $\bm{\theta}$ and $\mathbf{R}^{\prime} :=
\diag(\mathbf{a})^{-1}\mathbf{R}\diag(\mathbf{a})^{-1}$ is the
``normalized'' correlation matrix. Without loss of generality we can
assume that each component of $\mathbf{a}$ is strictly positive.
\begin{table}[t]
\vspace*{-2ex}
\caption{Summary of related work on separabile topic models.}
\label{relatedworks}
\centering
\begin{tabular}{|c|c|c|c|c|}
\hline 
{\bf Reference} & {\bf Condition} & {\bf
  Consistency ($N$ fixed,} & {\bf Sample} & {\bf Computational} \\
& {\bf on} $\mathbf{R}^{\prime}$ & $M \rightarrow\infty$) {\bf
  proved?} & {\bf complexity?} & {\bf complexity} \\
\hline 
\cite{ARORA:ref} & Simplicial & Yes & Yes & Poly., impractical \\
\hline 
\cite{recht2012factoring} & Simplicial & No & No & Poly., practical \\
\hline 
\cite{Kumar13:ref} & Simplicial & No & No & Poly., practical \\ 
\hline
\cite{Arora2:ref} & Full-rank & Yes & Yes & Poly., practical  \\
\hline
\cite{DDP:ref} & ``Diagonal dominance'' & Yes & Yes & Poly., practical
\\
\hline
\end{tabular}
\vglue -2ex
\end{table}
Among these additional conditions, the simplicial condition
(cf.~Sec.~\ref{sec:sim}) on 
$\mathbf{R}^{\prime}$ is the weakest sufficient condition for
consistent topic recovery that is available in the literature.
However, the existing approaches either lack statistical guarantees or
are computationally impractical. Algorithms with {\it both}
statistical and computational merits have been developed by imposing
{\it stronger} conditions as in \cite{Arora2:ref, DDP:ref}. Hence the
natural questions that arise in this context are:
\begin{enumerate}
\vspace*{-1ex}
\item[(a)] What are the {\bf necessary and sufficient} conditions for
  separable topic recovery?
\item[(b)] Do there exist algorithms that are consistent,
  statistically efficient, and computationally practical under these
  conditions?
\vspace*{-1ex}
\end{enumerate}
In this paper, we first show that the simplicial condition on the
normalized correlation matrix $\mathbf{R}^{\prime}$ is an
algorithm-independent, information-theoretic necessary condition for
consistently detecting novel words in separable topic models. The key
insight behind this result is that if $\mathbf{R}^{\prime}$ is
non-simplicial, we can construct two distinct separable topic models
with different sets of novel words which induce the same distribution
on the observations $\mathbf{X}$.
In Sec.~\ref{sec:rp}, we answer the second question in the affirmative
by outlining a random projection based algorithm and providing its
statistical and computational complexity. Due to space limitations,
the details of this novel algorithm and the proofs of the claimed
sample and computational complexity will appear elsewhere.
\section{Simplicial Condition}
\vspace{-2ex}
\label{sec:sim}
Similar to \cite{ARORA:ref}, we have
\begin{definition}
\label{def:sim}
A matrix $\mathbf{A}$ is simplicial if $~\exists \gamma >0$ such that each row of
$\mathbf{A}$ is at a Euclidean distance of at least $\gamma$
from the convex hull of the remaining rows.
\end{definition}
For estimating a separable topic matrix $\bm{\beta}$, the simplicial
condition is imposed on the normalized second order moment
$\mathbf{R}^{\prime}$. More precisely:
\begin{condition} \label{Simpicial_condition} (Simplicial Condition)
The topic model is simplicial if ${\mathbf R}^{\prime}$ is
simplicial, i.e., $\exists \gamma > 0$ s.t. every row  of ${\mathbf R}^{\prime}
$ is at a distance of at least $\gamma$ from the convex hull of
the remaining rows of ${\mathbf R}^{\prime}$.
\end{condition}
Algorithms with provable performance guarantees that exploit the
separability condition typically consist of two steps: (i) novel word
detection and (ii) topic matrix estimation. 
We will only focus on the detection of all novel words since the
detection problem is in itself important in many applications, e.g.,
endmember detection in hyperspectral, genetic, and metabolic datasets,
and also because the second estimation step is relatively easier once
novel words are correctly identified.
%
Our first main result is contained in the following lemma:
\begin{lemma} (Simplicial condition is necessary)
Let ${\bm \beta}$ be separable and $W > K$. If there exists an
algorithm that can consistently identify all the novel words of all
the topics,
then its normalized second order moment $\mathbf{R}^{\prime}$ is
simplicial.
\end{lemma}
\begin{proof}
The proof is by contradiction. We will show that if
$\mathbf{R}^{\prime}$ is non-simplicial, we can construct two topic
matrices $\bm{\beta}_1$ and $\bm{\beta}_2$ whose sets of novel words
are not identical and yet $\mathbf{X}$ has the same distribution under
both models. This will imply the impossibility of consistent novel
word detection.

Suppose $\mathbf{R}^{\prime}$ is non-simplicial. Then we can assume,
without loss of generality, that its first row is within the convex
hull of the remaining rows, i.e., $\mathbf{R}_1^{\prime} =
\sum_{j=2}^{K}c_j \mathbf{R}_j^{\prime}$, where
$\mathbf{R}_j^{\prime}$ denotes the $j$-th row of
$\mathbf{R}^{\prime}$, and $c_2,\ldots,c_K \geq 0$,
$~\sum_{j=2}^{K}c_j =1$ are convex weights. Compactly,
$\mathbf{e}^{\top}\mathbf{R}^{\prime} \mathbf{e} = 0$ where
$\mathbf{e} := \left[ -1, c_2, \ldots, c_K \right]^{\top}$.
Recalling that $\mathbf{R}^{\prime} =
\diag(\mathbf{a})^{-1}\mathbf{R}\diag(\mathbf{a})^{-1}$, where
$\mathbf{a}$ is a positive vector and $\mathbf{R} = \eE(\bm{\theta}_i
     {\bm{\theta}_i}^{\top} )$ with $\bm{\theta}_i$ denoting any
     column of $\bm{\theta}$, we have
\begin{eqnarray*}
0 = \mathbf{e}^{\top}\mathbf{R}^{\prime} \mathbf{e} 
= (\diag(\mathbf{a})^{-1}{\mathbf{e}})^{\top} \eE(\bm{\theta}_i
{\bm{\theta}_i}^{\top}) (\diag(\mathbf{a})^{-1}{\mathbf{e}})
= \eE ( \Vert {\bm{\theta}_i}^{\top}
\diag(\mathbf{a})^{-1}{\mathbf{e}} \Vert^2),
\end{eqnarray*}
which implies that $\Vert {\bm{\theta}_i}^{\top}
\diag(\mathbf{a})^{-1}{\mathbf{e}} \Vert \stackrel{a.s.} = 0$.
%
%
From this it follows that if we define two non-negative row vectors
$\mathbf{b}_1 := b\left[ a_1^{-1}, 0,\ldots,0 \right]$ and
$\mathbf{b}_2 = b \left[(1-\alpha) a_1^{-1}, \alpha c_2
  a_2^{-1},\ldots, \alpha c_K a_K^{-1}\right]$, where $b > 0, 0 <
\alpha < 1$ are constants, then $\mathbf{b}_1 \bm{\theta}_i
\stackrel{a.s.} = \mathbf{b}_2 \bm{\theta}_i$.

Now we construct two separable topic matrices $\bm{\beta}_1$ and
$\bm{\beta}_2$ as follows. Let $\mathbf{b}_1$ be the first row and
$\mathbf{b}_2$ be the second in $\bm\beta_1$. Let $\mathbf{b}_2$ be
the first row and $\mathbf{b}_1$ the second in $\bm\beta_2$. Let
$\mathbf{B}\in\rR^{W-2 \times K}$ be a valid separable topic
matrix. Set the remaining $(W-2)$ rows of both $\bm{\beta}_1$ and
$\bm{\beta}_2$ to be $\mathbf{B}(I_K -
\diag(\mathbf{b}_1+\mathbf{b}_2))$. We can choose $b$ to be small
enough to ensure that each element of $(\mathbf{b}_1+\mathbf{b}_2)$ is
strictly less than $1$. This will ensure that $\bm{\beta}_1$ and
$\bm{\beta}_2$ are column-stochastic and therefore valid separable
topic matrices. Observe that $\mathbf{b}_2$ has at lease two non-zero
components. Thus, word 1 is novel for $\bm\beta_1$ but non-novel for
$\bm\beta_2$.

By construction, $\bm{\beta}_1\bm{\theta} \stackrel{a.s.} =
\bm{\beta}_2 \bm{\theta}$, i.e., the distribution of $\mathbf{X}$
conditioned on $\bm{\theta}$ is the same for both models.
Marginalizing over $\bm{\theta}$, the distribution of $\mathbf{X}$
under each topic matrix is the same. Thus no algorithm can distinguish
between $\bm{\beta}_1$ and $\bm{\beta}_2$ based on $\mathbf{X}$. 
\end{proof} 

Our second key result is the sufficiency of the simplicial condition
for novel word detection:
\begin{lemma}
\label{lem:sufficient}
Assume that topic matrix ${\bm\beta}$ is separable. If
$~\mathbf{R}^{\prime}$ is simplicial, then there exists an
algorithm whose running time is at most quadratic in $W, M, K, N$,
that only makes use of empirical word co-occurrences and consistently
recovers the set of all novel words for $K$ topics as $M\rightarrow
\infty$.
\end{lemma}
This is a consequence of Lemma~\ref{thm:rp} in Sec.~\ref{sec:rp} where
an algorithm based on random projections is described that can attain
the claimed performance.

We conclude this section with two conditions that each imply the
simplical condition.
\begin{proposition}
\label{Simplicial_claims}
Let $\mathbf{R}^{\prime}$ be the normalized topic correlation matrix
matrix. Then,
(i) $\mathbf{R}^{\prime}$ is diagonal dominant, i.e., $~\forall i,j,
i\neq j$, $\mathbf{R}_{i,i}^{\prime} - \mathbf{R}_{i,j}^{\prime} >0$
$\Longrightarrow$ $\mathbf{R}^{\prime}$ is simplicial.
(ii) $\mathbf{R}^{\prime}$ is full rank $\Longrightarrow$
$\mathbf{R}^{\prime}$ is simplicial.
Furthermore, the reverse implications in (i) and (ii) do not hold in
general.
\end{proposition}
The proof of the above proposition is omitted due to space limitations
but is straightforward.  This demonstrates that the diagonal dominant
condition in \cite{DDP:ref} and the full-rank condition in
\cite{Arora2:ref} are both stronger than the simplicial condition.

\vspace{-2ex}
\section{Random Projection Algorithm}
\vspace{-1ex}
\label{sec:rp}
The pseudo-code of an algorithm that can achieve the performance
claimed in Lemma~\ref{lem:sufficient} is provided below
(cf.~Algorithm~\ref{Alg:RP}). Due to space limitations, we only
explain the high-level intuition which is geometric.
Let $\widetilde{\mathbf{X}}$ and $\widetilde{\mathbf{X}}^{\prime}$ be
obtained by first splitting each document into two independent copies
and then scaling the rows to make them row-stochastic.
The key idea is that if $\mathbf{R}^{\prime}$ is simplicial, then as
$M \rightarrow \infty$, with high probability, the rows of
$\widetilde{\mathbf{X}}^{\prime}\widetilde{\mathbf{X}}^{\top}$
corresponding to novel words will be extreme points of the convex hull
of {\it all} rows.
This suggests finding novel words by projecting the rows onto an
isotropically distributed random direction, $P$ times, and then
selecting the $K$ rows which maximize the projection value most
frequently.
%
%
%
%
%
\begin{algorithm}
\caption{Random Projection Algorithm for Novel Words Detection}
\label{Alg:RP}
\begin{algorithmic}[1]
\Require $\widetilde{\mathbf X}$, $\widetilde{\mathbf X}^{\prime}$,
$d$, $K$, $P$ \Comment{$d$ : some model constant; assumed known for
  simplicity}
\Ensure Set of novel words $\mathcal{I}$ 
\State $\mathbf{C} \leftarrow M \widetilde{\mathbf X}^{\prime}\widetilde{\mathbf X}^{\top}$
\State $\forall i$, $\mbox{Nbd}(i) \leftarrow \lbrace j: C_{i,i} -2C_{i,j}+C_{j,j} \geq d/2 \rbrace$,   \Comment{Exclude novel words of the same topic as $i$}
\For {$r=1,\ldots,P$} \Comment{Random Projections}
\State Sample ${\mathbf u}_r \sim \mathcal{N}(\mathbf{0},\mathbf{I}_W)$
	\State $\hat{p}_i^{(r)} \leftarrow \mathbb{I}\lbrace\forall j
        \in \mbox{Nbd}(i) : {\mathbf C}_i {\mathbf u}_r \geq {\mathbf
          C}_j {\mathbf u}_r \rbrace$ , $i=1,\ldots, W$ \Comment{The
          max. projected on $\mathbf{u}_r$}
\EndFor
\State $\hat{p}_i \leftarrow \frac{1}{P} \sum_{r = 1}^{{P}}
\hat{p}_i^{(r)}$, $i=1,\ldots, W$ \Comment{Freq. of being max.}
\State $k \leftarrow 1$, $\mathcal{I} \leftarrow$ $\arg\max_w \hat{p}_w$ and $i \leftarrow 2$
\While {$k \leq K$} \Comment{Extract top frequent maximums}
\State $j \leftarrow$ the index of the $i^{\text{th}}$ largest value
of $(\hat{p}_1, \ldots, \hat{p}_W)$
\If {$ j \in \bigcap_{l\in\mathcal{I}} \mbox{Nbd}(l) $ }
\State $\mathcal{I} \leftarrow \mathcal{I} \cup \{j\}$, $~~~k
\leftarrow k + 1$ \Comment{Not the novel words of the same topic as $i$}
\EndIf
\State $i \leftarrow i + 1$
\EndWhile
\end{algorithmic}
\end{algorithm}
We summarize the statistical and computational properties of this
algorithm in Lemma~\ref{thm:rp}:
\begin{lemma}\label{thm:rp}
Let topic matrix $\bm\beta$ be separable and $\mathbf{R}^{\prime}$ be
simplicial.
%
%
Then Algorithm~\ref{Alg:RP} will output all novel words of all $K$
topics consistently as $M \rightarrow \infty$ and $P\rightarrow
\infty$. Furthermore, $\forall \delta >0$, for
\begin{equation*}
M \geq \max \Biggl\{ c_1 \frac{\log(3W/\delta)}{ d^2 \phi^2 \eta^4}, ~
c_2\frac{W^2 \log (2W/q_{\wedge}) \log(3W/\delta)}{ \rho^2
  q_{\wedge}^2 \phi^2 \eta^4} \Biggr\} \mbox{~and~} \numofproj \geq
c_3 \frac{\log(3W/\delta)}{q_{\wedge}^2}
\end{equation*}
Algorithm~\ref{Alg:RP} fails with probability at most $\delta$, where
$c_1$ to $c_3$ are some absolute constants and $d, \phi, \eta$, and
$q_{\wedge}$ are constants that depend on model parameters $\bm\beta$,
$\mathbf{a}$, and $\mathbf{R}^{\prime}$.
Moreover, the running time of Algorithm \ref{Alg:RP} is
$\mathcal{O}(MN\numofproj + W\numofproj +K^2)$.
\end{lemma}
%
%
%
As summarized in Lemma~\ref{thm:rp}, the computational complexity of
Algorithm~\ref{Alg:RP} is linear in terms of $M, N, W$ and quadratic
in terms of $K$, which typically stays fixed. This is more efficient
than the best known provable algorithms in
\cite{ARORA:ref,Arora2:ref,DDP:ref}. The sample complexities for $M$
and $P$ are both polynomial in terms of $W, \log(\delta)$, and other
model parameters, which are comparable to the current state-of-the-art
approaches.
%
%
It turns out that Algorithm~\ref{Alg:RP} is also amenable to
distributed implementation since it only involves aggregating the
counts of rows that maximize projection values.

\vspace{-2ex}
\section{Discussion}
\vspace{-2ex}
The necessity of the simplicial condition we proved in this paper is
information-theoretic and algorithm-independent. It is also a
sufficient condition. Although widely used priors, e.g., Dirichlet,
satisfy the stronger full-rank and diagonal dominant conditions, in
certain types of datasets, e.g., Hyperspectral Imaging, these may not
hold~\cite{DDP:ref}.
This paper only focused on {\it detecting} distinct novel words of all
topics.
%
%
In general, the simplicial condition is not sufficient for
consistently {\it estimating} the topic matrix. It can be shown that
the full-rank condition is sufficient but not necessary.

\vspace*{-10pt}
\bibliographystyle{unsrt}
\footnotesize
\bibliography{SuffandNessCond_short}

\end{document}